\documentclass{ieeetran}
\usepackage{amsfonts}
\usepackage{amsmath}
\usepackage{graphicx}
\usepackage{color,soul}
\usepackage{multirow}
\usepackage{soul}
\usepackage{pdflscape}
\usepackage{lscape}
\usepackage{rotating}
\usepackage{booktabs}
\usepackage{placeins}
\usepackage{xr}
\usepackage{amsthm}
\usepackage{bbding}
\usepackage{array}
\usepackage{relsize}

\usepackage{cite}
\usepackage{amssymb}
\usepackage{esint}
\usepackage[colorlinks=true, allcolors=blue]{hyperref}
\usepackage{arydshln}
\usepackage{eurosym}
\usepackage{float}
\usepackage{lineno}
\usepackage{algorithm}
\usepackage{algorithmic}
\usepackage{booktabs}
\usepackage{subcaption}
\usepackage{multirow}
\usepackage{breqn}

\newcommand{\CASE}[1]{\STATE \textbf{case} #1\textbf{:} \begin{ALC@g}}
\newcommand{\ENDCASE}{\end{ALC@g}}

\newcommand{\DEFAULT}{\STATE \textbf{default:} \begin{ALC@g}}
\newcommand{\ENDDEFAULT}{\end{ALC@g}}
\newcommand{\DEFAULTLINE}[1]{\STATE \textbf{default:} }

\newtheorem{theorem}{Theorem}
\newtheorem{corollary}{Corollary}[theorem]
\newtheorem{lemma}[theorem]{Lemma}

\usepackage[font={scriptsize,it}]{caption}

\IEEEoverridecommandlockouts                              

\begin{document}

\title{\LARGE Adversarial Online Learning with Variable Plays in the Pursuit-Evasion Game: Theoretical Foundations and Application in Connected and Automated Vehicle Cybersecurity}

\author{
    Yiyang~Wang~\IEEEmembership{Graduate Student Member,~IEEE}, Neda~Masoud
    \thanks{Published in: IEEE Access, vol. 9, pp. 142475-142488, 2021, doi: 10.1109/ACCESS.2021.3120700.}
    \thanks{
        Yiyang Wang and Neda Masoud are with the University of Michigan, Ann Arbor, MI 48109,  USA (e-mail: yiyangw@umich.edu,  nmasoud@umich.edu).
    }
}





\maketitle

\begin{abstract}
We extend the adversarial/non-stochastic multi-play multi-armed bandit (MPMAB) to the case where the number of arms to play is variable. The work is motivated by the fact that the resources allocated to scan different critical locations in an interconnected transportation system change dynamically over time and depending on the environment. By modeling the malicious hacker and the intrusion monitoring system as the attacker and the defender, respectively, we formulate the problem for the two players as a sequential pursuit-evasion game. We derive the condition under which a Nash equilibrium of the strategic game exists. For the defender side, we provide an exponential-weighted based algorithm with sublinear pseudo-regret. We further extend our model to heterogeneous rewards for both players, and obtain lower and upper bounds on the average reward for the attacker. We provide numerical experiments to demonstrate the effectiveness of a variable-arm play.
\end{abstract}


\section{Introduction}
\label{sec:introduction}
Currently, the world is experiencing an evolution from the traditional transportation system to the next generation of intelligent transportation systems (ITS). ITS aims to satisfy the ever-increasing need for mobility in major cities, which has caused growing traffic congestion, air pollution, poor user experience and crashes. 
Developing a sustainable intelligent transportation system requires better usage of the existing infrastructures and their seamless integration with information and communication technologies (ICT). Enabled by the recent findings in the areas of telecommunications, electronics, and computing capabilities in recent decades, the subsystems (infrastructures and vehicles) in ITS are expected to interoperate and communicate with each other, in order to provide a better and safer traveling experience \cite{ibanez2015integration}.

The interconnection between the infrastructures and the vehicles relies on various types of sensors to provide state information and situational awareness. However, this has also increased the vulnerability of these advanced systems to cyber attacks. For instance, recently there have been demonstrated cyber attacks on vehicle sensors in \cite{cao2019adversarial, cao2019adversarial1}, where the authors used optimization-based approaches to fool the light detection and Ranging (LiDAR) sensors on the vehicle. At the system level, the infrastructures and the vehicles can be viewed as individual nodes in a large interconnected network, where a single malicious attack on a subset of sensors of one node can easily propagate through this network, affecting other network components (e.g., other vehicles, traffic control devices, etc.). For example, Feng et al. \cite{feng2018vulnerability} demonstrated that by sending falsified data to actuated and adaptive signal control systems, a malicious hacker could increase the total system delay in a real-world corridor. Therefore, there is an increasing need for cyber security solutions, especially for sensor security solutions, to enhance the safety and reliability of the entire system.

Cyber security is an extremely broad topic. However, previous work on cyber security in the realm of ITS mainly focuses on either attack or the defense strategies. For instance, there exists a large body of research illustrating the potential risks of connected and automated vehicle (CAV) technologies that result in anomalous/false information \cite{petit2014potential,checkoway2011comprehensive,weimerskirch2015overview,yan2016can}. In the case of CAV sensor security, several critical sensors are illustrated in \cite{parkinson2017cyber}, including differential global positioning systems (GPS), inertial measurement units, engine control sensors, tyre-pressure monitoring systems (TPMS), LiDAR, and camera. Meanwhile, CAVs require more engine control units (ECUs) and many features of CAVs require complex interactions between multiple ECUs, which may potentially expose more vulnerabilities compared to non-CAVs. There also exist several studies assessing the potential threats on the transportation infrastructure \cite{feng2018vulnerability,fok2013introduction,kelarestaghi2018intelligent}. For example, field devices such as traffic signals and roadside units are susceptible to tampering. The aforementioned literature illustrates the potential threats of sensor attacks to connected transportation systems. 

Besides threat detection, prevention is normally recognized as one of the best defense strategies against malicious hackers or attackers. In order to deploy better prevention mechanisms, behaviors of both the attacker and the defender have to be considered so that the attack profile can be predicted.
There is a gap in the literature in considering both the attacker and the defender and the adaptive interactions between them when devising defense strategies, which this paper aims to bridge. 

Moreover, as more sensors are mounted aboard CAVs or installed on the transportation infrastructure, it becomes more difficult to monitor the sensors continuously, mainly due to limited resources. Although there is a large body of literature addressing sensor security in ITS \cite{van2019real,wang2020real,wang2020anomaly,marchetti2016evaluation,muter2011entropy}, most of them mainly focus on sensor intrusion/anomaly detection without attack profile analysis, which considers which sensor is more vulnerable and should be protected. In this study, we address this by modeling attacker and defender behaviors in a game theoretical framework. Specifically, instead of considering intrusion/anomaly detection for all sensors in the system, we model attack and defense behaviors in order to predict which subset of sensors are more likely to be compromised. To be more practical, we consider a dynamic resource constraint for the defender. We model this problem as a sequential evasion-and-pursuit game between two players. Consider the intrusion monitoring system of a sensor network as the defender. At each time, the defender selects a subset of sensors to scan, while the number of selected sensors changes based on the environment and scanning history, among other factors. Meanwhile, a hacker, considered as the attacker, attempts to select a sensor to compromise without being scanned by the defender. We assume that both the attacker and the defender are able to learn their opponent's behavior adaptively and with only partial information over time, and investigate the the resulting decision problem.

The main contributions of this work are as follows: First, in order to predict the attack profile, we model the behaviors of the attacker and the defender as the adversarial (or non-stochastic) multi-armed bandit (MAB) problem and the multi-armed bandit problem with variable plays (MAB-VP), where the two players are playing a constant-sum game against each other. To the best of our knowledge, this is the first study of MAB-VP in the non-stochastic setting. 
Second, we derive conditions under which a Nash equilibrium of the strategic game exists. For the defender, we provide an exponential-weighted algorithm, which is shown to have sublinear pseudo-regret. 
Finally, we consider a more realistic setting where the rewards are heterogeneous among different sensors, and derive lower and upper bounds on the attacker's average reward.

\section{Literature Review}

In this paper, we explore online learning algorithms in the class of adversarial or non-stochastic multi-armed bandit (MAB) problems. The adversarial MAB problem was first addressed by Auer et al. \cite{auer1995gambling}, where they also proposed the well-known exponential-weight algorithm for exploration and exploitation (Exp3). 
Exp3 runs the Hedge algorithm, which was originally proposed by Freund and Schapire \cite{FREUND1997119} 
as a subroutine.
Since then, there have been several extensions to this class including the online shortest path problem \cite{gyorgy2007line}, routing games \cite{nisanalgorithmic}, bandit online linear optimization \cite{DBLP:conf/colt/AbernethyHR08}, and combinatorial bandits \cite{cesa2012combinatorial}.

The multi-play multi-armed bandit (MPMAB) problem is another research direction for MAB. In this extension, a fixed number of resources (i.e., arms) are allocated at each time step. The MPMAB has attracted a lot of interest and several studies have been conducted along this direction \cite{anantharam1987asymptotically, agrawal1990multi,komiyama2015optimal,xia2016budgeted}. However, most of these studies only focus on a stochastic setting. There is much less concentration on the adversarial MPMAB problem: Cesa-Bianchi and Lugosi \cite{cesa2012combinatorial} considered combinatorial bandits in the adversarial setting, where they proposed the ComBand algorithm. This algorithm has a sublinear regret in $O\left(M^{\frac{3}{2}}N\sqrt{TN \ln{N}}\right)$, with time and space complexities of $O\left(MN^3\right)$ and $O\left(K^3\right)$, respectively, where $M$ is the number of resources (or arms selected) at each time, $T$ is the number of iterations, and $N$ is the number of possible actions.
Following this work, Uchiya et al. \cite{10.1007/978-3-642-16108-7_30} proposed an extension of Exp3, Exp3.M, which runs in $O\left(N (\log{M}+1)\right)$ time and $O(N)$ space, and suffers at most $O\left(\sqrt{MTN\log{(N/M)}}\right)$ regret. However, the aforementioned algorithms only consider a fixed number of arms to be played at each time.

Only a limited number of studies have considered variable plays. Fouché et al. \cite{fouche2019scaling} proposed a scaling algorithm combined with a MAB algorithm , which they call the S-MAB algorithm. In this algorithm, the number of arms played at each time changes in order to satisfy an efficiency constraint. However, although the authors considered a dynamic environment, the S-MAB algorithm uses a stochastic setting, where they assume an unknown distribution of reward for each arm. Another work addressing the variable plays problem was done by Lesage-Landry and Taylor \cite{lesage2017multi}, where they extended the stochastic MAB to stochastic plays setting, i.e. the number of arms to play evolves as a stationary process. Both these studies only considered a stochastic setting, and did not conduct any game strategy analysis.

Although there is a wealth of research on using game theory in the transportation literature, very few studies applied game theory in ITS cybersecurity. Sedjelmaci et al. \cite{sedjelmaci2019cyber} conducted a survey on recent studies utilizing game theory to protect ITS from attacks, which is to the best of our knowledge the only survey paper on this topic. However, without considering the adaptive behavior of opponents, the current literature mostly models the cybersecurity problem as a non-repeated game, such as the Stackelberg security games (SSG) \cite{kiekintveld2009computing,sinha2015physical}, zero-sum games \cite{alpcan2010security, mejri2016new}, or Bayesian games \cite{bahamou2016game,sedjelmaci2016intrusion}. The solutions from these types of models are typically in the form of equilibria with an implied assumption that the players have knowledge of their opponent's actions/beliefs. Instead, we formulate this cybersecurity problem as a sequential pursuit-evasion game, which is also in the realm of algorithmic learning theory. 
There have been several studies of the pursuit-evasion problem \cite{vidal2002probabilistic, navda2007using, wang2015learning}. However, they either lack robustness against adaptive changes in the adversarial behavior, or do not consider multiple plays, variable plays, dynamic resource allocation, or heterogeneous rewards.

Since the behavior of the adversarial opponent usually cannot be described in a stochastic way, in this paper we study the MAB-VP problem in a  non-stochastic setting, where we propose the Exp3.M with variable plays (Exp3.M-VP) algorithm. Next, we consider a game setting for two players, and show that a Nash equilibrium of the strategic game exists. Finally, we consider heterogeneous rewards for both players and derive lower and upper bounds for the attacker's average reward. Numerical analyses are conducted in order to further demonstrate our results.

\section{System Model and Problem Formulation}

\begin{table*}[htbp]\caption{Table of Notation}
\begin{center}
\begin{tabular}{r c p{10cm} }
\toprule
$\alpha_k(t)/\beta_k(t)$ & $\triangleq$ & marginal probability that the attacker compromises/the defender scans location $k$ at time $t$\\
$x_k(t)/y_k(t)$ & $\triangleq$ & indicator variable of whether the defender/attacker selects the location $k$ at time $t$\\
$I_t/J_t$ & $\triangleq$ & index of the locations where the attacker compromises/the defender scans at time $t$\\
$M_t$ & $\triangleq$ & number of locations scanned by the defender at time $t$\\ 
$a/b$ & $\triangleq$ & lower/upper bound of $M_t$\\ 
$r(t)/s(t)$ & $\triangleq$ & single step reward of the attacker/defender\\  
$\omega(t)/\theta(t)$ & $\triangleq$ & private randomization device of the attacker/defender\\ 
$\pi_t/\gamma_t$ & $\triangleq$ & control policy of the attacker/defender\\ 
$T$ & $\triangleq$ & finite time horizon\\ 
$N$ & $\triangleq$ & total number of locations\\ 
$\mathcal{N}$ & $\triangleq$ & index set of $N$ locations\\ 
$\mathcal{C}$ & $\triangleq$ & index set of arbitrary locations\\ 
\bottomrule
\end{tabular}
\end{center}
\label{tab:TableOfNotationForMyResearch}
\end{table*}

\subsection{System Model}
Consider the repeated pursuit-evasion game between an attacker and a defender in discrete time. At each time step $t$, the attacker selects one of the $N$ locations, indexed by the set $\mathcal{N} =\{1,2,...,N\}$, to hide in (e.g., compromise a sensor), while the defender searches $M_t$ locations simultaneously, where $1 \leq a \leq M_t\leq b < N$. The behaviors of the attacker and the defender are described by their respective set of marginal probabilities $\alpha(t) = (\alpha_k(t))_{k\in \mathcal{N}}$ and $\beta(t) = (\beta_k(t))_{k\in \mathcal{N}}$, where $\alpha_k(t)$ and $\beta_k(t)$ are the respective probabilities that the $k$-th location is chosen by the attacker and the defender at time $t$. Note that $\alpha(t)$ and $\beta(t)$ represent the adversarial behavior with respect to one's opponent at time $t$, where they can describe randomized strategies of the players, or a probabilistic belief held by one side about the likelihood of an action by the other side.

Define two sets of binary variables $x_k(t)$ and $y_k(t)$ such that $x_k(t) = 1$ if the defender does not search location $k$ at time $t$, and $x_k(t) = 0$ otherwise. Similarly, $y_k(t) = 1$ if the attacker compromises the location $k$ at time $t$, and $y_k(t) = 0$ otherwise. When the attacker (defender) does not know the type of algorithm/strategy the opponent uses, it may regard the $x_k(t)$ ($y_k(t)$) as a predetermined but unknown number. When the attacker (defender) does have this information, it may regard the $x_k(t)$ ($y_k(t)$) as a random variable, where $P\left(x_k(t) = 0\right) = \beta_k(t)$ (resp. $P\left(y_k(t) = 1\right) = \alpha_k(t)$). The game is played in a sequence of trials $t = 1,2,...,T$. 
In this work we consider the case that neither the attacker nor the defender knows the strategy adopted by the other player. As will be discussed later, they have to choose the location based on the their history rewards.

\subsection{Problem Formulation: Partial Information Game}
In this study we consider the scenario where both players have limited information on the adaptive behavior of their opponent. Define $\pi = (\pi_t,t=1,2,...)$ as the control policy of the attacker, and let $\Pi$ denote the policy space. Denote the location selection (action) sequence as $I = (I_t, t=1,2,...)$ under policy $\pi$ and $|I_t| = 1$. At each time and under policy $\pi_t$, the attacker chooses one location $I_t \in \mathcal{N} $ to attack, i.e.,

\begin{linenomath}\begin{equation}
I_t = \pi_t\left(x_{I}^{[t-1]},I^{[t-1]},\omega(t)\right),  
\end{equation}\end{linenomath}

\noindent where $x_{I}^{[t-1]} :=(x_{I_1}(1),...,x_{I_{t-1}}(t-1))$, and $I^{[t-1]}$ is similarly defined. $(\omega(t), t=1,2,...)$ denotes the randomized strategy of the attacker. Let $x_k(t)$ be the state of location $k$ for the attacker at time $t$.
Then the attacker scores the corresponding reward $r^{I}(t) =x_{I_t}(t)$. The attacker observes only the reward $r^{I}(t) $ for the chosen action $I_t$. 

The attacker receives an expected reward $E[r^{I}(t)] = 1 - \beta_{I_t}(t)$ at time $t$, which is the mean number of successful attacks at the chosen location. Note that in this section we consider a homogeneous reward across all locations; however, heterogeneous location-dependent rewards are considered in section \ref{sec hetoro reward}. In this study, we assume a 100\% success rate for both attacks and detection attempts. 
Then, within the time window $\{t, t=1,2,...,T\}$, the attacker considers the following maximization problem,

\begin{linenomath}\begin{equation}
    \underset{\pi\in \Pi, I_t\in \mathcal{N}}{\textrm{maximize}} \ \mathbb{E}\left\{ \sum_{t=1}^{T} x_{I_t}(t)   \right\},
\end{equation}\end{linenomath}

\noindent where the expectation is with respect to the randomness of the system state and the mixed-strategy of the attacker.

We assume that the defender can scan $M_t$ locations at time $t$. Define $\gamma = (\gamma_t,t=1,2,...)$ as the control policy of the defender, and let $\Gamma$ denote the defender's policy space. Denote the location selection (action) sequence as $J = (J_t, t=1,2,...)$ under policy $\gamma$. At each time and under policy $\gamma_t$, the defender scans $M_t$ locations, denoted as set $J_t \subset \mathcal{N}$ and $|J_t| = M_t$, based on their history search and rewards, i.e.,

\begin{linenomath}\begin{equation}
J_t = \gamma_t\left(y_{J}^{[t-1]},J^{[t-1]}, M_t,\theta(t)\right),  
\end{equation}\end{linenomath}

\noindent where $y_{J}^{[t-1]} :=(y_{J_1}(1),...,y_{J_{t-1}}(t-1))$ with $J^{[t-1]}$ similarly defined, and $(\theta(t), t=1,2,...)$ denotes the randomized strategy of the defender.
Let $y_k(t)$ be the state of location $k$ for the defender at time $t$. The defender also observes only the rewards $ \sum_{j\in J_t}y_{j}(t)$ of the selected action $J_t$. Denote the total rewards at time $t$ of the defender given the location selection sequence $J$ as $s^{J}(t) =\sum_{j\in J_t}y_j(t)$. 
The defender therefore receives the expected reward  $E[s^{J}(t)] = \sum_{j\in J_t} \alpha_j(t)$ at time $t$. This expected reward represents the mean number of detected attacks among $M_t$ number of scanned locations.

We assume that the number of arms $M_t$ the defender plays at each time is determined by a scaling function, i.e. $f: \mathbb{R}^{N+1} \xrightarrow{} \{a,a+1,...,b\}$, of the $d$-moving average of the rewards of each arm, where $a$ and $b$ are integers, and $1\leq a \leq b < N$. We also assume that $M_t$ is a function of the environment constraint $L_t$, since in reality checking a location (e.g., scanning a specific sensor/unit in a CAV) may consume resources.
Then, given the time horizon $T$, the defender is trying to solve the following constrained optimization problem:

\begin{linenomath}
\postdisplaypenalty=0
\begin{subequations}
\label{eq3}
\begin{align}
     \underset{\gamma\in \Gamma, J_t\subset \mathcal{N}}{\textrm{maximize}} \thickspace &\mathbb{E}\left\{ \sum_{t=1}^{T} \sum_{j\in J_t} y_{j}(t)   \right\} \\
     \textrm{s.t.} \thickspace &M_t = f(\hat{y}^d(t),L_t)\\
     &|J_t| = M_t
\end{align}
\end{subequations}
\end{linenomath}

\noindent where $\hat{y}^d(t) := (\hat{y}_1^d(t),\hat{y}_2^d(t),...,\hat{y}_N^d(t))$, and $\hat{y}_i^{d}$ is the $d$-moving average of the rewards of each arm $i$. Using a moving average of reward can allow us to capture the history reward while at the mean time capturing the dynamic change of the reward for each location, allowing the scaling function to adjust the number of arms to play each time. The expectation is with respect to the randomness of the system state and the mixed strategy of the defender. Note that there is no requirement for the scaling function $f$, other than it needs to be bounded by integers $a$ and $b$. Furthermore, $L_t$ can be an arbitrary integer between $a$ and $b$, thereby capturing any set of environmental conditions.

When the defender knows the type of strategy the attacker uses, it may regard $y_j^{J}(t)$ as stochastic, i.e. assuming the attacker chooses location $j$ with probability $P(y_j^{J}=1) = \alpha_j(t)$. 
Note that this is different from the stochastic MAB setting where a fixed (time-invariant) distribution of rewards for each arm is assumed. 
However, here we do not assume neither the defender nor the attacker have information about their opponent's strategy. Hence, the difficulty is that the defender can only estimate $\alpha_j(t)$ by imposing an arbitrary belief on the adversarial behavior based on previous observations and rewards. Furthermore, here, we do not make any assumptions about the distribution of $\alpha_j(t)$. 



\begin{algorithm}
\small
\caption{\textbf{Exp3.M-VP}}
\begin{algorithmic}[1]
\label{alg:1}
\STATE Parameter: $\eta \in (0,1]$
\STATE Initialization: $w_i(1) =1$ for $i=1,2,...,N$
\FOR{ $t=1,2,...,T$}
    \STATE Receive the number of arms to play at each round $M_t$.
    \IF{$\max_{j\in \mathcal{N}}w_j(t) \geq \left(\frac{1}{M_t}-\frac{\eta}{N}\right)\sum_{i=1}^{N}w_i(t)/(1-\eta)$}
    \STATE Decide $\kappa_t$ such that
    $$\frac{\kappa_t}{\sum_{w_i(t)\geq \kappa_t}\kappa_t +\sum_{w_i(t)< \kappa_t }w_i(t)} = \left(\frac{1}{M_t}-\frac{\eta}{N}\right)/(1-\eta).$$
    Set $S_0(t) = \{i:w_i(t)\geq \kappa_t\}$. \\
    Set $w'_i(t) = \kappa_t, \forall i\in S_0(t)$.
    \ELSE
    \STATE Set $S_0(t) = \emptyset$.
    \ENDIF
    \STATE Set
    $
    w'_i(t) = w_i(t), \forall i\in S_0^c(t).
    $
    \STATE Set
    $
    \hat{\alpha}_i(t) = M_t \left((1-\eta) \frac{w'_i(t)}{\sum_{j=1}^N w'_J(t)} + \frac{\eta}{N}\right).
    $
    \STATE Set
    $
    J_t = \textrm{\textbf{DepRound}}(M_t,(\hat{\alpha}_1,\hat{\alpha}_2,...,\hat{\alpha}_N)).
    $
    \STATE Observe rewards $y_i(t) \in [0,1]$ for $i\in J_t$.
    \FOR{$i=1,2,...,N$}
    \STATE 
    \[
    \hat{y}_i(t) =
    \begin{cases}
    y_i(t)/\hat{\alpha}_i(t) & \textrm{if}\ i\in J_t, \\
    0 & \textrm{otherwise}.
    \end{cases}
    \]
    \[
    w_i(t+1) =
    \begin{cases}
    w_i(t)\exp(M_t \: \eta \: \hat{y}_i(t)/N) &\text{if} \ i \in S^c_0(t), \\
    w_i(t) &\text{otherwise.}
    \end{cases}
    \]
    \ENDFOR
\ENDFOR
\end{algorithmic}
\end{algorithm}

\section{Algorithms for the Attacker and the Defender}
We assume the attacker adopts Exp3 proposed by Auer et al \cite{auer1995gambling}. (However, as we are going to show later in section \ref{sec attacker}, the equilibrium of the two-player game does not depend on any proprieties of the algorithm other than a no-regret guarantee.) The Exp3 algorithm uses an efficient and randomized policy to select only one arm at each time $t$. The adversarial single play bandit problem is closely related to the problem of learning to play an unknown repeated matrix game. In this setting, a player without prior knowledge of the game matrix is to play the game repeatedly against an adversary with complete knowledge of the game and unbounded computational power. The basic idea of Exp3 is that at each time the player uses a randomized policy such that the adversarial player cannot know the exact choice of the player before she/he plays. For the details of Exp3, refer to the Appendix \ref{apendx1}.

Unlike the attacker who selects a single location to attack, we assume the defender can search multiple number of locations, which may vary at each time. Both sides seek to maximize their respective total rewards. 
At the beginning of a time step, each side needs to decide which location(s) to target, and cannot change their selection until the next time step. We develop a variable-play extension of the Exp3.M algorithm for the defender, which we call Exp3.M-VP, as detailed in Algorithm \ref{alg:1}. In the Exp3.M-VP algorithm, let $S$ denote the set of selected locations, and let $S^c$ define its complement set. Under the non-stochastic assumption and at each time step, the Exp3.M-VP algorithm consists of the following two procedures:

\begin{enumerate}
    \item Receive $M_{t}$, which is determined by the scaling function $f$ and could be based on the environment constraint $L_t$ as well as the history rewards $\hat{y}^d(t)$ at time $t$, among other factors. Note that function $f$ can take any form, and defining its exact form is outside the scope of this paper. Here, we assume $M_t$ is provided.

    \item Apply an adversarial MPMAB algorithm which selects $M_t$ arms (locations) to play.
\end{enumerate}

For the second procedure, 
we use the Exp3.M algorithm as a subroutine of the Exp3.M-VP algorithm.
The Exp3.M is proposed by Uchiya et al. \cite{10.1007/978-3-642-16108-7_30} and is an extension of the algorithm Exp3 for the adversarial MPMAB setting. In contrast to the Exp3 algorithm which selects one arm at each time, Exp3.M randomly selects a fixed number of $M$ arms at each time. Note that both Exp3 and Exp3.M suffer from sublinear (weak) regret, or no-regret. In order to make sure that the probability of selecting location $i$ by DepRound at step 12, i.e. $\hat{\alpha}_i(t)$, does not exceed 1, 
the Exp3.M-VP algorithm checks whether all $w_j(t)$'s are less than $\left(\frac{1}{M_t}-\frac{\eta}{N}\right)\sum_{i=1}^{N}\frac{w_i(t)}{(1-\eta)}$ at step 5. If that is the case, $\hat{\alpha}_i(t)$ calculated at step 11 will be less than 1 for all $i = 1,2,...,N$ without any weight modification, and the set $S_0(t)$ is set to $\emptyset$ at step 8. Otherwise, all the actions $i$ with $w_i(t)\geq \kappa_t$ are classified into $S_0(t)$ and set to $\kappa_t$ at step 6. Doing this, we have  $\hat{\alpha}_i(t) = 1$ for all $i \in S_0(t)$. The subroutine \textbf{DepRound} \cite{gandhi2006dependent} at step 12 draws $M_t$ out of $N$ items with the specified marginal distribution $(\hat{\alpha}_1,\hat{\alpha}_2,...,\hat{\alpha}_N)$, and is included in Appendix \ref{appendix2}.

\section{Adaptive Learning of the Defender}

In this section, we address the adaptive learning of the defender. Based on Algorithm \ref{alg:1} for the defender, the problem (\ref{eq3}) can be recast by removing the constraint set, since will divide the problem to a scaling procedure and the MAB-VP. Formally, let $\mathbf{y}(t) := (y_k(t), \forall k\in \mathcal{N})$ for $t = 1,...,T$ over a finite horizon $T$. For any search sequence of the defender $J = (J_t,t=1,2,...)$ and a fixed sequence of attacks by the attacker $(\mathbf{y}(1),\mathbf{y}(2),...)$, the total reward of the defender at $T$, denoted by $G^J(T)$, is given by

\begin{linenomath}\begin{equation}
    G^J(T) = \sum_{t=1}^{T}\sum_{j\in J_t}y_j(t).
\end{equation}\end{linenomath}

\noindent

Here, we obtain the maximum reward by consistently searching the subset $\mathcal{A}_{M_t}$, which is the most attacker-active location set at each time step $t$ with cardinality $M_t$:

\begin{linenomath}\begin{equation}
    G_{\max}(T) = \underset{\substack{\mathcal{A}_{M_t}}}{\max} \sum_{t=1}^{T}\sum_{k\in \mathcal{A}_{M_t}} y_k(t).
\end{equation}\end{linenomath}

\noindent
Let us define $\mathcal{A}=\cup_{M_t} \mathcal{A}_{M_t}$, where $\mathcal{A}\subset \mathcal{N}$. Note that if $M_t \in \{a,a+1,...,b\}$, the location index subset $\mathcal{A}_{M_t}$ is defined such that $\mathcal{A}_{a}\subset \mathcal{A}_{a+1} \subset... \subset \mathcal{A}_{b} = \mathcal{A}$.

The regret is then defined as 

\begin{linenomath}
\begin{equation}
R(T) = G_{\textrm{max}}(T) -  G^J(T).
\end{equation}
\end{linenomath}

\noindent
When $a = b$, i.e. $M_t$ is time-invariant, the above regret reduces to the standard regret of MPMAB problem.

Since we care more about the competition against the optimal action in expectation, we define the pseudo-regret for our MAB-VP problem following the definition of pseudo-regret in \cite{bubeck2012regret} as:

\begin{linenomath}\begin{equation}
    \bar{R}(T) = G_{\max}(T) - E[G^J(T)],
\end{equation}\end{linenomath}

\noindent
where the expectation is with respect to the randomness of the system state and the mixed-strategy of the defender.


\begin{theorem}
\label{THE3.1}
 For any $N >0$ and for any $\eta \in (0,1]$, if $M_t$ is lower bounded and upper bounded by two positive integers $a$ and $b$ respectively, then

 \begin{multline}
     \Bar{R}_{\textrm{Exp3.M-VP}}(T) = G_{\max}(T) - E[G^J_{\textrm{Exp3.M-VP}}(T)] \\
     \leq \left(1+\frac{(e-2)b}{a} \right) \eta G_{\max}(T) + \frac{N}{\eta}\ln\frac{N}{b}
 \end{multline}


\noindent 
holds for any assignment of rewards and for any $T>0$.
 
\end{theorem}

\begin{proof}
See Appendix \ref{appendix3}.
\end{proof}

By appropriately choosing the parameter $\eta$, we can obtain the following corollary:
\begin{corollary}
\label{cor511}
Set $\eta = \min\left\{1, \sqrt{
\frac{N a  \ln(N/b)}{(a+(e-2)b)b T}}\right\}$. Then

\begin{linenomath}
\[
\Bar{R}_{\textrm{Exp3.M-VP}}(T) 
\leq 2\sqrt{\left(1 + (e-2)\frac{b}{a}\right)}\sqrt{bTN\ln \frac{N}{b}}
\]
\end{linenomath}

\noindent
holds for any $T>0$ and for any assignment of rewards.
\end{corollary}

 

The proof of Corollary \ref{cor511} is the same as that of Corollary 3.2 in \cite{auer2002nonstochastic}. 
For the proof of Corollary \ref{cor511}, see Appendix \ref{appendix cor511}.
Note that when $a = b$, the upper bound in Corollary \ref{cor511} is the same as the upper bound of Exp3.M in \cite{10.1007/978-3-642-16108-7_30}, and when $a = b = 1$ the upper bound becomes the same upper bound obtained for Exp3 in \cite{auer1995gambling}.

\begin{corollary}
\label{cor 312}
Define $\bar{s}_{\infty}:= \underset{T\rightarrow \infty}{\lim\inf} E\left[\frac{1}{T}\sum_{t=1}^{T}s^{J}(t)\right]$ as the average reward of the defender over infinite time horizon. Using the same parameter $\eta$ as in Corollary \ref{cor511}, when the defender uses the Exp3.M-VP algorithm against the attacker who adopts a no-regret algorithm, we have
$
\bar{s}_{\infty} = \frac{\nu}{N}
$
if $M_t$ is a wide sense stationary process with mean $\nu$.
\end{corollary}

In order to prove Corollary \ref{cor 312}, we need the following lemma,  which was originally derived in \cite{wang2015learning}.

\begin{lemma}
\label{lem3.2}
When the defender (pursuer) is adopting Exp3.M and the attacker (evader) does not know the type of algorithm used by the adversarial opponent, then $v = \frac{1}{N}$, where $v$ is the game value of the repeated constant-sum game for the defender.
\end{lemma}

Then the proof of Corollary \ref{cor 312} is as follows.
\begin{proof}
The above problem is equivalent to the problem of two players playing an unknown repeated bimatrix game, where the game value $v_{i,t}$ ($i=1,2$ for the row and column player respectively) is changing over time. Define the game matrices as two $N\times N$ matrices $\mathbf{B}$ and $\mathbf{C}$, where $B_{ij}+C_{ij} = 1$ for any $(i,j) \in \mathcal{N}\times \mathcal{N}$. At each time $t$, the defender (i.e., the row player) chooses $J_t$ rows of the matrix, and at the same time, the attacker (i.e., the column player) chooses exactly one column $I_t = k$. The defender then receives the payoff $\sum_{j\in J_t} B_{jk} =  \sum_{j\in J_t} y_j(t)$. The defender uses a mixed strategy $\mathbf{p}_t$ at each time $t$, where $\mathbf{p}_{t} \in [0, 1]^{N}$, and the attacker chooses according to a probability vector $\mathbf{q}_t \in[0, 1]^N$. Note that the sum of $\mathbf{p}_t$ equals $M_t$ and the sum of $\mathbf{q}_t$ equals 1. Let $v_{1,t}$ be the game value of the game matrix $\mathbf{B}$ at time $t$. 
Then by Corollary \ref{cor511}, we have

\begin{linenomath}
\begin{subequations}\label{eq14a}
\begin{alignat}{2}
    E\left[\sum_{t=1}^T \sum_{j\in J_t} B_{jk}\right] &=&& E\left[\sum_{t=1}^T \sum_{j\in J_t} y_j(t)\right]\\
    &\geq&& G_{\max}(T) - 2\sqrt{\left(1 + (e-2)\frac{b}{a}\right)}\notag\\
    &&&\times\sqrt{bTN\ln \frac{N}{b}}.
\end{alignat}
\end{subequations}
\end{linenomath}

Let $\mathbf{p}_t$ be such that 
$$v_{1,t} = \max_{\mathbf{p}_t}\min_{\mathbf{q}_t}\mathbf{p}^\intercal_t \mathbf{B} \mathbf{q}_t = \min_{\mathbf{q}_t}\max_{\mathbf{p}_t}\mathbf{p}^\intercal_t \mathbf{B} \mathbf{q}_t.$$
Then we have 

\begin{linenomath}\begin{subequations}\label{eq15}
    \postdisplaypenalty=0\begin{flalign}
        G_{\max}(T) &\geq \sum_{t=1}^T
        \sum_{i=1}^N p_{t,i}y_i(t)\\
        &= \sum_{t=1}^T \mathbf{p}_t^\intercal \mathbf{y}(t) \\
        &= \sum_{t=1}^T \mathbf{p}_t^\intercal \mathbf{B} \mathbf{q}_t \geq \sum_{t=1}^T v_{1,t}
\end{flalign}\end{subequations}
\end{linenomath}

\noindent
where $\mathbf{q}_t$ is a distribution vector whose $I_t$-th component is 1.

Combining (\ref{eq14a}) and (\ref{eq15}), we have

\begin{linenomath}
    \postdisplaypenalty=0\begin{flalign}\label{eq16}
        E\left[\frac{1}{T}\sum_{t=1}^{T}s^{J}(t)\right] \geq &\frac{1}{T} \sum_{t=1}^T v_{1,t} -
        2\sqrt{\left(1 + (e-2)\frac{b}{a}\right)}\notag\\
        &\times\sqrt{bN\ln \frac{N}{b}/T}.
    \end{flalign}
\end{linenomath}

Note that at each time $t$, $v_{1,t} = M_t v_1$, where $v_1$ is the game value when the defender only chooses one location. 
Hence, by taking the limit of (\ref{eq16}) and according to the law of large numbers we have
\begin{linenomath}\begin{equation}
    \bar{s}_{\infty}= \underset{T\rightarrow \infty}{\lim\inf} \frac{1}{T}\sum_{t=1}^{T}v_{1,t} = \nu v_1,
\end{equation}\end{linenomath}
where the first equality comes from the fact that the attacker is also adopting a no-regret algorithm (e.g. Exp3).
Finally, according to Lemma \ref{lem3.2}, we obtain the result.

\end{proof}

\begin{corollary}
\label{cor512}
     Under the setting that the defender adopts Exp3.M-VP and the attacker adopts a
     no-regret
     algorithm, assuming that $M_t$ is a wide sense stationary process with mean $\nu$, each player adopts the best response for the infinite-horizon problem.
\end{corollary}

The proof can be obtained by extending the proof of the defender side in Corollary \ref{cor 312} to both sides, and is omitted for brevity. Note that in Corollary \ref{cor 312} 
and Corollary \ref{cor512} 
we do not specify which type of learning algorithm the attacker is using, and the only assumption is that the attacker adopts a no-regret algorithm.



\section{Adaptive Learning of the Attacker}
\label{sec attacker}

We assume that the attacker adopts the Exp3 algorithm to randomly attack one location at each time step. 
The Exp3 algorithm runs the algorithm Hedge as a subroutine. 
Unlike the Hedge algorithm which directly takes advantage of the full information of the reward vector $\mathbf{x}(t) := (x_i(t), \forall i \in \mathcal{N}) $, Exp3 observes partial information and feeds the simulated reward vector $\hat{\mathbf{x}}(t) := (\hat{x}_i(t), \forall i \in \mathcal{N})$ to the Hedge. The Hedge will then update $\hat{\beta}_i(t)$, which is the prediction of probability $\beta_i(t)$ for $i \in \mathcal{N}$. For more details about the Exp3 and Hedge algorithms, see Appendix \ref{apendx1}.

The defender adopts the Exp3.M-VP algorithm, which has a sublinear regret, as shown in Theorem \ref{THE3.1}. As a result, if the attacker favors one location, intuitively the defender will eventually identify this most attractive location, and fails to scan it only at a rate no more than sublinear in $T$. When $M_t$ is a time-invariant constant, it follows immediately that the best strategy for the attacker over an infinite time horizon is to treat each location equally, either in a stochastic or deterministic way. However, when $M_t$ is a variable, the same argument cannot be trivially made.
\begin{theorem}
\label{THE4.1}
 Define $\Bar{r}_{\infty} := \underset{T\rightarrow \infty}{\lim\inf} E\left[\frac{1}{T}\sum_{t=1}^{T}r^{I}(t)\right]$, and let the location sequence $g$ be the sequence of the greedy policy $\pi_{\textrm{greedy}}$, where $g(t) = \arg \min_{i\in \mathcal{N}}\hat{\beta}_i(t)$ for all $t$. If $M_t$ is bounded by two positive integers $a,b$ such that $M_t \in \{a,a+1,...,b\}$, then under any policy $\pi$ we have:

\begin{linenomath}
 \[
 \begin{aligned}
 \Bar{r}_{\infty} \leq \frac{N-a}{N},
 \end{aligned}
 \]
 \end{linenomath}

\noindent
 and under the greedy policy $\pi_{greedy}$, 
 
  \begin{linenomath}
 \[
 \Bar{r}_{\infty} \geq \frac{N-b}{N}.
 \]
\end{linenomath}

\end{theorem}

\begin{proof}
See Appendix \ref{appendix4}.
\end{proof}

Note that by Corollary \ref{cor 312}, we can directly obtain the following result,
\begin{corollary}
\label{cor4.1}
Under the setting that the defender adopts Exp3.M-VP, the attacker adopts Exp3, and $M_t$ is a wide sense stationary process with mean $\nu$, we have $\Bar{r}_{\infty} = \frac{N-\nu}{N}$.
\end{corollary}

Moreover, when $M_t$ is a wide sense stationary process, following the proof of Theorem \ref{THE4.1}, it is not hard to show that even the greedy policy can obtain $\Bar{r}_{\infty} = \frac{N-\nu}{N}$. Note that the above argument does not require Exp3.M-VP to have any property other than a no-regret guarantee, and therefore the greedy policy for the attacker can be a countermeasure against the entire family of no-regret algorithms. For the defender part, according to Corollary \ref{cor 312} and Corollary \ref{cor4.1}, a straightforward path to increase the average reward in an infinite time horizon is to increase the value of $\nu$, i.e., assign more resources to the intrusion monitoring system.






\section{Adaptive Adversarial Learning with Heterogeneous Rewards}\label{sec hetoro reward}
In this section we consider heterogeneous rewards that are location-dependent. This corresponds to a more general setting, 
since in reality some locations (e.g., sensors) are more critical to the system than others. Let $\mu_k$ be the location-dependent reward corresponding to the $k$-th location. That is, the rewards of the attacker and the defender are $r^{I}(t) = \mu_{I_t} x_{I_t}(t)$ and $s^{J}(t) = \sum_{j \in J_t} \mu_{j} y_{j}(t)$, respectively. 
Without loss of generality, we assume that $\mu_1 \geq \mu_2\geq ... \geq \mu_N$. We denote the frequency of location $k$ being selected given the selection sequence $I$ as $d_k^{I}(T)$ over a time horizon $T$, i.e.,

\begin{linenomath}\begin{equation}
    \frac{1}{T}\sum_{t=1}^T  \mu_{I_t} = \frac{1}{T} \sum_{k=1}^N c^{I}_k(T)\mu_k = \sum_{k=1}^{N} d_k^{I}(T)\mu_k
\end{equation}\end{linenomath}

\noindent
where $c^{I}_k(T) = |\left\{t\leq T: I_t = k\right\}|$ and $d_k^{I}(T) = c^{I}_k(T)/T$. Note that $c^{I}_k(T)$ is the total number of times location $k$ is selected by the attacker over horizon $T$ given the selection sequence $I$. 

Since the problem is no longer a constant-sum game under the setting of heterogeneous rewards, Corollary \ref{cor512} and Corollary \ref{cor4.1} cannot be directly applied.
However, we can still show that when the reward for each location is heterogeneous, the average reward $\bar{r}_{\infty}$ in an infinite time horizon is bounded within an interval determined by  $a$, $b$, and $\mu_k, k = 1,2,..., N$.

\begin{theorem}
\label{the5.1}
 Given heterogeneous rewards, the average reward of the attacker $\bar{r}_{\infty}$ over an infinite time horizon is bounded within the interval $\left[\frac{K^* - b}{\sum_{k=1}^{K^*}\mu_k},\frac{K^* - a}{\sum_{k=1}^{K^*}\mu_k}\right]$, where $K^*$ is a constant determined by $\mu_k$ values such that $b\leq K^*\leq N$.
\end{theorem}

In order to prove Theorem \ref{the5.1}, we need Lemmas \ref{LEMMA5.2} and \ref{LEMMA5.3}, as follows. Let $\textrm{supp}(\mathbf{d}) = \{k\in \mathcal{N}: d_k>0\}$ for any feasible solution $\mathbf{d}$, and let $K^*$ be the cardinality of $\textrm{supp}(\mathbf{d})$. Then we have the following lemmas: 
\begin{lemma}
\label{LEMMA5.2}
For any optimal solution $\mathbf{d}^*$ of problem (\ref{eq21}), ($i$) $\mu_k d_k^* = \mu_j d_j^*$ for any $k,j \in \textrm{supp}(\mathbf{d}^*)$, and ($ii$), $\textrm{supp}(\mathbf{d}^*)$ consists of the indices of locations with the $K^*$ highest $\mu$.
\end{lemma}

\begin{lemma}
\label{LEMMA5.3}
Problem (\ref{eq26}) is lower bounded by $\frac{K^*-b}{\sum_{k=1}^{K^*}\mu_k}$.
\end{lemma}

The proofs of Lemmas \ref{LEMMA5.2} and \ref{LEMMA5.3} can be found in the Appendices \ref{appendix5} and \ref{appendix6}, respectively.

\noindent
Now we shall give the proof of Theorem \ref{the5.1} as follows.
\begin{proof}
The average reward of the defender when using Exp3.M-VP is given by

\begin{linenomath}\begin{subequations}
\postdisplaypenalty=0\begin{flalign}
     E[G^J_{\textrm{Exp3.M-VP}}(T)] &= E\left[\sum_{t=1}^{T} \sum_{j\in J_{t}} \mu_j y_j(t) \right] \\
     &= \sum_{t=1}^T\sum_{k=1}^N \mu_k y_k(t)\beta_{k}(t)\\
     &= \sum_{t=1}^T \mu_{I_t} \beta_{I_t}(t)\\
     &= \sum_{t=1}^T \mu_{I_t} - E\left[\sum_{t=1}^T r^I(t)\right]
\end{flalign}    
\end{subequations}
\end{linenomath}

\noindent
for any realization $I$.

Then we have

\begin{linenomath}\begin{subequations}
\begin{alignat}{2}
    \frac{1}{T} E\left[\sum_{t=1}^T r^{I}(t)\right] &=&& \frac{1}{T} \sum_{t=1}^T \mu_{I_t} - \frac{1}{T}E[G^J_{\textrm{Exp3.M-VP}}(T)] \\
    &\leq&& \sum_{k=1}^{N} \mu_k d_k^{I}(T) -\notag\\
    &&&\frac{1}{T}\left(G_{\textrm{max}}(T) - 2\sqrt{\left(1 + (e-2)\frac{b}{a}\right)}\right.\notag\\
    &&&\times\left.\sqrt{bTN\ln \frac{N}{b}}\right)\\
    &\leq&& \sum_{k=1}^{N} \mu_k d_k^{I}(T) - \underset{J\in \mathcal{C}(\mathcal{N},a)}{\max}\sum_{j\in J} \mu_j d^I_j(T) \notag\\
    &&&+2\sqrt{\left(1 + (e-2)\frac{b}{a}\right)}\sqrt{bN\ln \frac{N}{b}/T}
\end{alignat}\end{subequations}
\end{linenomath}

\noindent
where $\mathcal{C}(\mathcal{N},a) = \left\{\mathcal{S} \subseteq \mathcal{N}:|\mathcal{S}| = a \right\}$, namely, the set of all subsets of size $a$ in $\mathcal{N}$. The second inequality uses the fact that

\begin{align*}
    G_{\max}(T) &\geq  \underset{\substack{J}\in \mathcal{C}(\mathcal{N},a)}{\max} \sum_{t=1}^{T}\sum_{j\in J} \mu_j y_j(t) \\
    &= \underset{\substack{J}\in \mathcal{C}(\mathcal{N},a)}{\max}\sum_{j\in J}\mu_j c^I_j(T).
\end{align*}

\noindent
Therefore, by having $T$ approach infinity, we have

\begin{linenomath}\begin{equation}
    \bar{r}_{\infty} \leq \underset{T\rightarrow \infty}{\lim\inf} E\left[\sum_{k=1}^{N} \mu_k d_k^{I}(T) - \underset{J\in \mathcal{C}(\mathcal{N},a)}{\max}\sum_{j\in J} \mu_j d^I_j(T) \right]
\end{equation}\end{linenomath}

\noindent
for any policy $\pi$.

Consider the following optimization problem

\begin{linenomath}\begin{equation}
\label{eq21}
    \underset{\mathbf{d} \in \Delta_N }{\text{maximize}}\  \sum_{k=1}^{N} \mu_k d_k - \underset{J\in \mathcal{C}(\mathcal{N},a)}{\max}\sum_{j\in J} \mu_j d_j,
\end{equation}\end{linenomath}

\noindent
where $\Delta_N$ is the set of distributions over $\mathcal{N}$ and $\mathbf{d} = (d_k, k\in\mathcal{N})$. Let the optimal solution and its objective function value be $\mathbf{d}^*$ and $r_{\max}$, respectively. Then we have

\begin{linenomath}\begin{equation}
    \bar{r}_{\infty} \leq r_{\max} = \sum_{k=1}^{N} \mu_k d^*_k - \underset{J\in \mathcal{C}(\mathcal{N},a)}{\max}\sum_{j\in J} \mu_j d^*_j.
\end{equation}\end{linenomath}

Without loss of generality, we assume that $supp(\mathbf{d}^*) = \{1,2,...,K^*\}$. Therefore, according to Lemma \ref{LEMMA5.2}, we have $d_k^* = \frac{1/\mu_k}{\sum_{j=1}^{K^*}1/\mu_j}$ for all $k\leq K^*$. Then the optimal value of problem (\ref{eq21}) is given by $(K^*-a)/\sum_{j=1}^{K^*}1/\mu_j$, which is increasing with respect to the value of $K^* = 1,2,...,N$. This gives the upper bound of $\bar{r}_{\infty}$.

When the defender adopts Exp3M-VP, we have

\begin{linenomath}\begin{equation}
    E[G^I_{\textrm{Exp3}}(T)] \geq G'_{\textrm{max}(T)} - o(T).
\end{equation}\end{linenomath}

\noindent
where $G^I_{\textrm{Exp3}}(T)$ is the total reward of the attacker when adopting Exp3, and $G'_{\max}(T) = \underset{k \in \mathcal{N}}{\max} \sum_{t=1}^T x_k(t)$ is the maximum total reward the attacker can gain when selecting a fixed location to attack. 

Similarly, define $h_k^J(T) = |\{t\leq T: k\in J_t \}|$ and $l_k^J(T) = h_k^J(T)/T$. Then, we have

\begin{linenomath}\begin{equation}
    G'_{\textrm{max}}(T) = \underset{k \in \mathcal{N}}{\max}\ \mu_k (T - h_k^J(T)).
\end{equation}\end{linenomath}

Thus, the average reward $\bar{r}_{\infty}$ of the attacker over an infinite time horizon is lower bounded by

\begin{linenomath}\begin{equation}
    \bar{r}_{\infty} \geq \underset{T\rightarrow \infty}{\lim\inf} E\left\{\underset{k \in \mathcal{N}}{\max}\ \mu_k (1 - l_k^J(T))\right\}.
\end{equation}\end{linenomath}

Consider the following optimization problem

\begin{linenomath}\begin{equation}
\label{eq26}
    \underset{c \in \Delta_N}{\textrm{minimize}} \ \underset{k\in \mathcal{N}}{\max}\ \mu_k (1-l_k),
\end{equation}\end{linenomath}

\noindent
and denote the optimal value of problem (\ref{eq26}) as $r_{\min}$. Then according to Lemma \ref{LEMMA5.3}, $r_{\min} = \frac{K^*-b}{\sum_{k=1}^{K^*}\mu_k}$, which gives us the lower bound of $\bar{r}_{\infty}$.

\end{proof}

\begin{figure}[tbh!]
    \centering
    \includegraphics[width=0.48\textwidth]{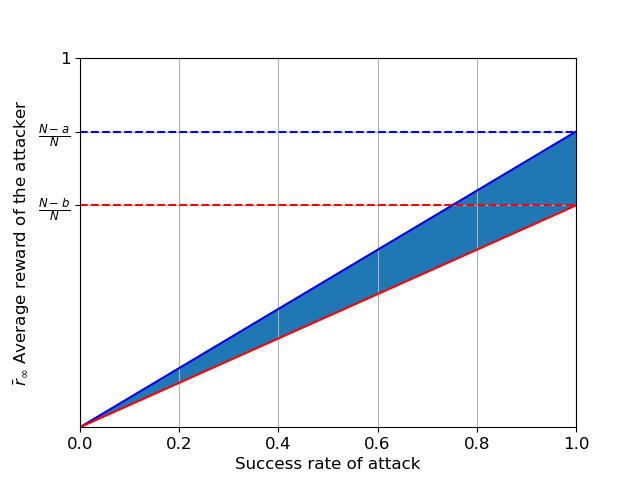}
    \caption{Range of the average reward of the attacker in an infinite time horizon under different attack success rates.}
    \label{fig:reward_vs_attack_rate}
\end{figure}

Theorem \ref{the5.1} is practical when the attack success rate of the attacker is not 100 percent for all locations, where $\mu_k$ represents the success rate of attacks on location $k$ for the attacker. Note that although Theorem \ref{the5.1} assumes heterogeneous rewards, it can be simply applied to homogeneous rewards as well. Figure \ref{fig:reward_vs_attack_rate} shows the range for the attacker's average reward in an infinite time horizon under different attack success rates, where we assume the same attack success rate for all locations for simpler visualization. Note that we do not even assume that $M_t$ is a wide sense stationary process; the only assumption here is that it is confined within a range with lower and upper bounds $a$ and $b$, respectively. The shaded blue region in Figure \ref{fig:reward_vs_attack_rate} indicates the potential reward the attacker can obtain in infinite time, and the red and blue lines indicate the lower and upper bounds on the attacker's average reward in infinite time, according to Theorem \ref{the5.1}. When the attack success rate is 1, the lower and upper bounds become equivalent to the bounds in Theorem \ref{THE4.1}. It is straightforward to see that the lower the success rate of the attack, the safer the system will be.

\begin{figure*}
    \centering
    \begin{subfigure}[tbh!]{0.49\textwidth}
    \centering
    \includegraphics[width=0.9\textwidth]{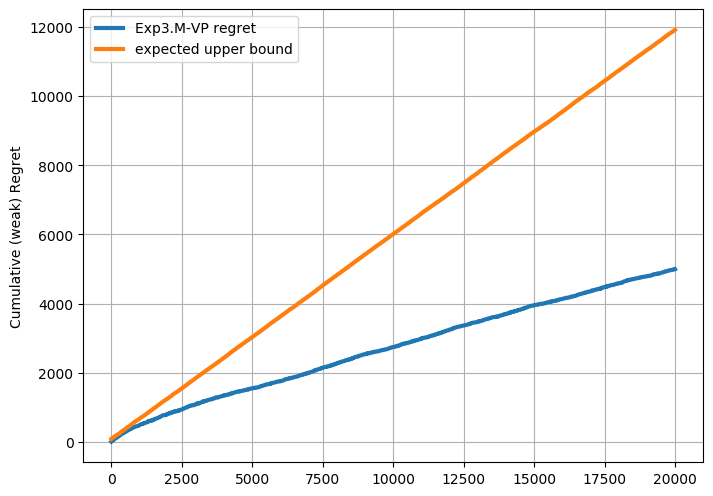}
    \caption{Exp3.M-VP regret (blue curve) and expected upper bound of regret (orange curve).}
    \label{fig1}
    \end{subfigure}
    \hfill
    \begin{subfigure}[tbh!]{0.49\textwidth}
    \centering
    \includegraphics[width=0.9\textwidth]{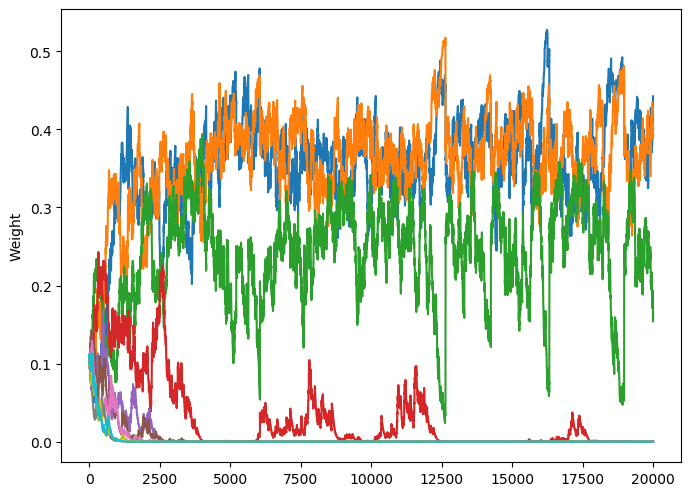}
    \caption{Normalized weights of 10 arms over 20,000 time steps.}
    \label{fig2}
    \end{subfigure}
    \caption{Simulation of Exp3.M-VP on a ten-armed bandit problem.}
\end{figure*}

\section{Numerical Analysis}
\label{sec na}
We conducted extensive simulations illustrating the performance of the proposed algorithm and policy. Our numerical analysis consists of three parts. In section \ref{sec na1}, we conduct simulations to test the Exp3.M-VP performance under a single-player setting. In section \ref{sec na_car_hacking}, we compare the performance of Exp3.M-VP with several bandit learning algorithms, i.e., the Exp3, Exp3.M, upper-confidence Bound (UCB) \cite{auer2002finite}, and $\epsilon$-greedy algorithms \cite{sutton1998introduction}, on real in-vehicle network datasets from the Car-Hacking datasets \cite{8514157}. In section \ref{sec na2}, we run simulations on the proposed game model and algorithmic solutions.

\subsection{Simulations on a Single Player}
\label{sec na1}
In this section we consider the single-player setting, where the Exp3.M-VP algorithm was evaluated on a ten-armed bandit problem with rewards for arms drawn independently from Bernoulli distributions with means $\{0.75,...,\frac{3}{4k},...,0.075\}$, with $k = 1,2,...,10$. This scenario was simulated over a fixed time horizon $T = 20,000$ time steps. The number of arms played at each time step is drawn independently from a discrete uniform distribution over $\{1,2,3\}$. Parameter $\eta$ is set to 0.1. 


Figure \ref{fig1} shows the regret of Exp3.M-VP versus the expected upper bound of the regret from Theorem \ref{THE3.1}. We can see that the actual regret of Exp3.M-VP has a smaller rate than its expected upper bound and the discrepancy becomes larger as time increases. Figure \ref{fig2} shows the change of the normalized weight for each location over the entire time horizon. As shown in this figure, Exp3.M-VP chooses the top three locations (i.e. the blue, orange, and green curves) with the highest average reward only after a short period of time, and the rest of weights vanish to nearly 0. The reason why only three locations pop up is that $M_t$, i.e. the number of the arms played at each time, is within the set $\{1,2,3\}$. The fluctuations of the weights are partly due to the fact that the Exp3.M-VP algorithm needs to explore different locations in order to update the choice prediction and estimation, and partly due to the fact that the sum of the weights must always equal to $M_t$, which is changing over time.

\subsection{Evaluations on Car-Hacking Dataset for the Defender}
\label{sec na_car_hacking}

In this section we compare Exp3.M-VP with Exp3, Exp3.M, UCB, and the $\epsilon$-greedy algorithms by implementing these algorithms over two in-vehicle network datasets from the Car-Hacking datasets. The Car-Hacking datasets are generated by logging the Controller Area Network (CAN) traffic via the OBD-II port from a real vehicle while message injection attacks were made. The Datasets each contain 300 intrusions of message injections over 26 unique CAN IDs. Each intrusion is performed for 3 to 5 seconds, and each dataset has a total of 30 to 40 minutes of the CAN traffic. Specifically, we test the performance on the spoofing attack datasets, which were conducted on the RPM gauze and the driving gear. That is, among 26 arms representing CAN IDs, two of them (RPM gauze and driving gear) contained spoofing attacks.

\begin{figure}[tbh!]
    \centering
    \includegraphics[width=0.48\textwidth]{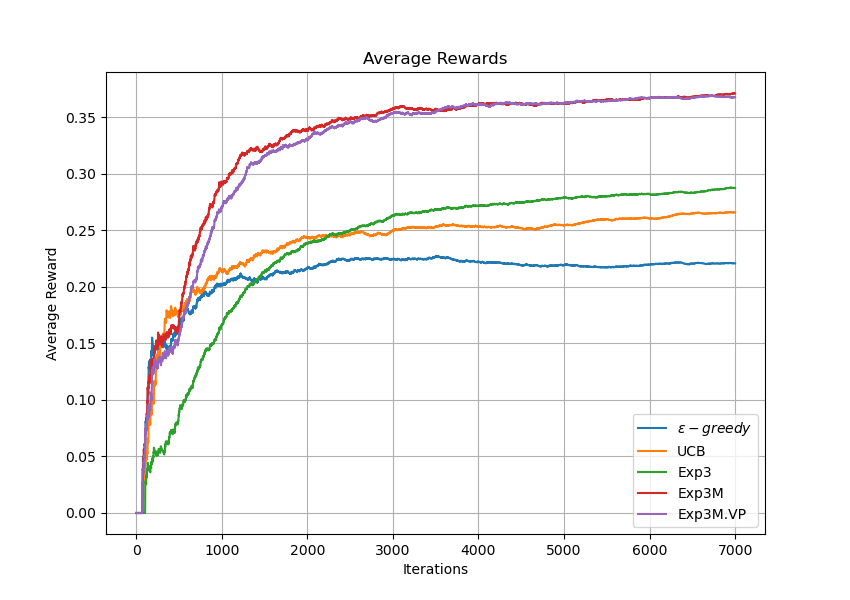}
    \caption{Cumulative average rewards for $\epsilon$-greedy, UCB, Exp3, Exp3.M, and Exp3.M-VP.}
    \label{fig_compare}
\end{figure}

Figure \ref{fig_compare} shows the cumulative average rewards for each bandit learning algorithm used by the defender. The experiments were conducted over $T = 7,000$ time steps, and the number of arms played by Exp3.M-VP was sampled from a truncated Gaussian distribution within the interval [1,3], with mean 2 and standard deviation 0.8. 
The number of arms played by Exp3.M was set to 3. 
We can see that both Exp3.M and Exp3.M-VP obtain higher cumulative average rewards than other single-play setting algorithms, due to the benefits from multiple or variable plays. 
Exp3.M-VP in this setting is a constrained version of Exp3.M, since the number of arms in Exp3.M (3) is an upper bound on the number of arms available to Exp3.M-VP ($[1,3]$). This indicates that Exp3.M-VP may have access to a smaller number of arms due to resource constraints. To make it more challenging, Exp3.M-VP does not know in advance the number of number of arms it may have access to in the future. 
Therefore, not surprisingly, Exp3.M obtains a slightly higher cumulative average reward than Exp3.M-VP. However, interestingly, eventually the cumulative average rewards of Exp3.M and Exp3.M-VP approach the same value. This demonstrates the power of the Exp3.M-VP algorithm: despite the fact that in average Exp3.M-VP plays fewer arms than Exp3.M, it can match the performance of Exp3.M. The reason is that only 2 out of of 26 CAN-IDs contained spoofing attacks, and after a period of time (i.e. around 3500 iterations), both Exp3.M and Exp3.M-VP are able to identify the top two most rewarded CAN-IDs.

\begin{figure}[tbh!]
    \centering
    \includegraphics[width=0.48\textwidth]{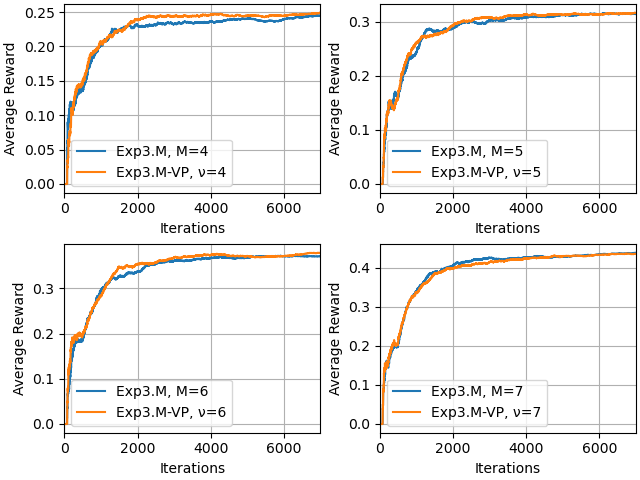}
    \caption{Average reward of Exp3.M and Exp3.M-VP under different $M_t$ and $\nu$.}
    \label{fig:compare arms}
\end{figure}

We further conduct sensitivity analysis on the number of arms played by Exp3.M and Exp3.M-VP. Specifically, we test the performance of the two algorithms with $M=\nu \in \{4,5,6,7\}$, where $M$ is the number of arms played by Exp3.M. For each $\nu$, we sample $M_t$ from a truncated Gaussian distribution within the interval $[\nu-1,\nu+1]$, with mean $\nu$ and standard deviation 0.8. As such, in this set of experiments the number of arms played by Exp3.M is the mean value of the number of arms played by Exp3.M-VP. Figure \ref{fig:compare arms} shows the results. This figure demonstrates the average reward of the two algorithms under four values for $M$ and $\nu$. We can see that the performance of the two algorithms are very close, mainly due to the fact that $M = \nu$. Note that here, in some instances Exp.M-VP will have access to less resources/arms, and in some instances more. As a result, throughout the iterations, sometimes Exp3.M outperforms Exp3.M-VP, and sometimes it underperforms. However, eventually both algorithms reach the same reward and successfully identify the attacked arms. This again demonstrates the strength of Exp3.M-VP, because the number of arms are determined exogenously and therefore Exp3.M-VP is able to match the reward obtained by Exp3.M under uncertainly on the number of available arms at each time.


\subsection{Simulations on Two Players}
\label{sec na2}
We now consider a game setting where two players, i.e., an attacker and a defender, are playing the pursuit-evasion game against each other. This corresponds to the realistic scenario where a malicious hacker is trying to compromise either the sensor/ECU in an in-vehicle sensor network, or the entire vehicle/infrastructure in an interconnected transportation system without being identified by the intrusion monitoring system. At the same time, the intrusion monitoring system is trying to identify as many compromised locations as possible to minimize the potential loss. We consider a ten-armed bandit problem for the two players, where the attacker adopts Exp3 and the defender adopts Exp3.M-VP. The scenario was simulated over $T = 100,000$ time steps, and the number of arms played by the defender was sampled from a truncated Gaussian distribution within the interval $[1,3]$, with mean 2 and standard deviation 0.8. The parameter $\eta$ for both Exp3 and Exp3.M-VP was set according to Corollary \ref{cor511}.

Figure \ref{fig3} illustrates the average reward and the equilibrium reward for the two players. Since we have $N = 10$ and $\nu = 2$, according to Corollary \ref{cor 312} and Corollary \ref{cor4.1}, the equilibrium rewards for the attacker and the defender are 0.8 and 0.2, respectively. We can see that the average rewards of both players converge to the equilibrium rewards after a relatively short period, and after that the average rewards stay around the equilibrium reward with small fluctuations. The fluctuations are due to the fact the Exp3 and Exp3.M-VP use randomized policies and need to occasionally explore different locations in order to update the choice predictions and estimations. 

\begin{figure}[tbh]
    \centering
    \includegraphics[width=0.4\textwidth]{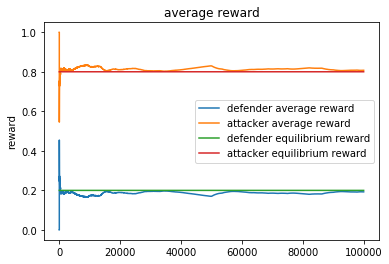}
    \caption{Average reward of the attacker and the defender over 100,000 time steps.}
    \label{fig3}
\end{figure}

\section{Conclusions}
In this paper, we extend the adversarial/non-stochastic MPMAB to the case where the number of plays can change in time, and propose the Exp3.M-VP algorithm for obtaining the variable-play property. This extension is motivated by the uncertainty of resources allocated to the intrusion monitoring system to scan at each time in resource-constrained systems, such as an interconnected transportation system. We derive a sublinear regret bound for Exp3.M-VP, which simplifies to the existing bounds in the literature when the number of arms played at each time is constant. We introduce a game setting where an attacker and a defender play a pursuit-evasion game against each other. The defender, who represents the intrusion monitoring system, adopts Exp3.M-VP and the attacker, who represents the malicious hacker, adopts Exp3. We derive the condition under which a Nash equilibrium of the strategic game exists. Finally, we consider heterogeneous rewards for arms, and obtain lower and upper bounds on the average rewards for the attacker in an infinite time horizon. We provide several numerical experiments that demonstrate our results.

This work provides insights on deploying an intrusion monitoring system either in an in-vehicle network or a transportation network: In order to minimize the potential loss of the system from cyber threats, one can either increase the average resources allocated to intrusion monitoring, or change the potential reward vector for each location to reduce the reward bound in Theorem \ref{the5.1}. One of the potential extensions of this work is to consider the connectivity or correlations between different arms, which can take into account the spread of the cyber attacks, and use such information to facilitate the decision making of the intrusion monitoring system.


\begin{appendices}

\section{Hedge and Exp3 Algorithms}
\label{apendx1}
\begin{algorithm}[H]
\small
\caption{\textbf{Hedge}}
\begin{algorithmic}[1]
\label{alg:3}
\STATE Parameters: $\iota \in \mathbb{R}^+$.
\STATE Initialization: Set $r_k(1):=0$ for all $k \in \mathcal{N}$.
\FOR{t = 1,2,...,T}
\STATE Choose action $I_t$ according to the distribution
\[
\beta_{k} = \frac{(1+\iota)^{r_k(t)}}{\sum_{j=1}^{N}(1+\iota)^{r_j(t)}}.
\]
\STATE Receive the reward vector $x(t)$ and score gain $x_{I_t}(t)$.
\STATE Set $r_k(t+1):= r_k(t) + x_k(t)$ for all $k\in\mathcal{N}$.
\ENDFOR
\end{algorithmic}
\end{algorithm}

\begin{algorithm}[H]
\small
\caption{\textbf{Exp3}}
\begin{algorithmic}[1]
\label{alg:4}
\STATE Parameters: $\iota \in \mathbb{R}^+$ and $\eta \in [0,1]$.
\STATE Initialization: Initialize \textbf{Hedge}.
\FOR{t = 1,2,...,T}
\STATE Obtain the distribution vector $\beta(t) = (\beta_k(t), k\in \mathcal{N})$ from \textbf{Hedge}.
\STATE Select action $I_t$ to be $k$ with probability 
\[
\hat{\beta}_k(t) = (1 - \eta)\beta_k(t) + \eta/N.
\]
\STATE Receive the reward $x_{I_t}(t) \in [0,1]$.
\STATE Return the simulated reward vector $\hat{x}(t) = (\hat{x}_k(t), k\in \mathcal{N})$ to \textbf{Hedge} with
\[
\hat{x}_k(t) = 
\begin{cases}
\frac{\eta}{N} \times \frac{x_{I_t}(t)}{\hat{\beta}_{I_t}(t)} &\text{if}\ k = I_t\\
0 &\text{otherwise.}
\end{cases}
\]
\ENDFOR
\end{algorithmic}
\end{algorithm}

\section{DepRound Algorithm}
\label{appendix2}
\begin{algorithm}[H]
\small
\caption{\textbf{DepRound:} The Dependent Rounding Algorithm}
\begin{algorithmic}[1]
\label{alg:2}
\STATE Inputs: Natural number $M< N$, marginal distribution $(p_k, k\in \mathcal{N} )$ with $\sum_{k=1}^N p_k = M$
\STATE Output: Subset $\mathcal{N}_1$ of $\mathcal{N}$ such that $|\mathcal{N}_1| = M$
\WHILE{$\{k\in \mathcal{N}: 0<p_k<1 \} \neq \emptyset$}
\STATE Choose distinct $i$ and $j$ such that $0<p_i<1$ and $0<p_j<1$
\STATE Set $\rho  = \min\{1-p_i,p_j\}$ and $\zeta = \min\{p_i,1-p_j\}$
\STATE Update $p_i$ and $p_j$ as
\[
(p_i,p_j) = 
\begin{cases}
(p_i + \rho,p_j - \rho)  &\text{with probability} \ \frac{\zeta}{\rho + \zeta} \\
(p_i - \zeta, p_j + \zeta) &\text{with probability} \ \frac{\rho}{\rho + \zeta} 
\end{cases}
\]
\ENDWHILE
\RETURN $\{k: p_k = 1, 1\leq k\leq N\}$
\end{algorithmic}
\end{algorithm}

\section{Proof of Theorem \ref{THE3.1}} 
\label{appendix3}
\begin{proof}
Let $W_t := \sum_{k=1}^{N}w_k(t)$ and $W_t' := \sum_{k=1}^{N}w_k'(t)$. Then, at each time step $t$,

\begin{linenomath}\begin{subequations}\label{eq10}
\postdisplaypenalty=0\begin{alignat}{2}
    \frac{W_{t+1}}{W_t} &=&& \sum_{i\in S^c_0(t)} \frac{w_i(t+1)}{W_t} + \sum_{i\in S_0(t)}\frac{w_i(t+1)}{W_t} \label{eq10.1}\\
    &=&& \sum_{i\in S^c_0(t)} \frac{w_i(t)}{W_t}\exp\left(\frac{\eta M_t}{N}\hat{y}_i(t)\right) + \sum_{i\in S_0(t)} \frac{w_i(t)}{W_t} \label{eq10.2}\\
    &\leq&& \sum_{i\in S^c_0(t)}\frac{w_i(t)}{W_t}\Bigg[1+\frac{\eta M_t}{N}\hat{y}_i(t) + (e-2)\notag\\
    &&&\times\left(\frac{\eta M_t}{N}\hat{y}_i(t)\right)^2\Bigg] + \sum_{i\in S_0(t)} \frac{w_i(t)}{W_t} \label{eq10.3}\\
    &=&& 1+\frac{W'_t}{W_t} \sum_{i\in S^c_0(t)}\frac{w_i(t)}{W'_t}\Bigg[\frac{\eta M_t}{N}\hat{y}_i(t) + (e-2)\notag\\
    &&&\times\left(\frac{\eta M_t}{N}\hat{y}_i(t)\right)^2\Bigg]  \label{eq10.4}\\
    &=&& 1+\frac{W'_t}{W_t} \sum_{i\in S^c_0(t)}\frac{\frac{\hat{\alpha}_i(t)}{M_t}-\frac{\eta}{N}}{1-\eta}\Bigg[\frac{\eta M_t}{N}\hat{y}_i(t)+(e-2) \notag\\
    &&& \times\left(\frac{\eta M_t}{N}\hat{y}_i(t)\right)^2\Bigg] \label{eq10.5}\\
    &\leq&& 1 + \frac{\eta}{(1-\eta)N} \sum_{i\in S^c_0(t)} \hat{\alpha}_i(t)\hat{y}_i(t) + \frac{(e-2)M_t \eta^2}{(1-\eta)N^2} \notag\\
    &&&\times\sum_{i\in S^c_0(t)} \hat{\alpha}_i(t)\hat{y}^2_i(t) \label{eq10.6}\\
    &\leq&& 1 + \frac{\eta}{(1-\eta)N}\sum_{i\in J_t\cap S^c_0(t)} y_i(t) + \frac{(e-2)M_t\eta^2}{(1-\eta)N^2}\notag\\
    &&&\times\sum_{i\in \mathcal{N}}\hat{y}_i(t).\label{eq10.7}
\end{alignat}\end{subequations}
\end{linenomath}

\noindent
Inequality (\ref{eq10.3}) uses $e^a \leq 1 + a + a^2$, $\forall a\in [0,1]$, equality (\ref{eq10.5}) holds because of step 11 in Algorithm \ref{alg:1}, inequality (\ref{eq10.6}) uses the fact that $\frac{W'_t}{W_t}\leq 1$, and the last inequality (\ref{eq10.7}) holds because $\hat{\alpha}_i(t)\hat{y}_i(t) = y_i(t)\leq 1$ for $i\in J_t$ and $\hat{\alpha}_i(t)\hat{y}_i(t) =0$ for $i\notin J_t$.
Then, according to inequality (\ref{eq10.7}) and by summing over $t$, we have

\begin{linenomath}\begin{subequations}\label{eq11}
    \begin{alignat}{2}
    \ln{\frac{W_{T+1}}{W_1}} &=&& \sum_{t=1}^T\ln\frac{W_{t+1}}{W_t} \label{eq11.1}\\
    &\leq&& \sum_{t=1}^T\ln\Bigg[1 + \frac{\eta}{(1-\eta)N}\sum_{i\in J_t\cap S^c_0(t)} y_i(t) \notag\\
    &&&+ \frac{(e-2)M_t\eta^2}{(1-\eta)N^2}\sum_{i\in \mathcal{N}}\hat{y}_i(t)\Bigg]\label{eq11.2}\\
    &\leq&& \frac{\eta}{(1-\eta)N} \sum_{t=1}^T \sum_{i\in J_t\cap S^c_0(t)} y_i(t) \notag\\
    &&&+ \frac{(e-2)b\eta^2}{(1-\eta)N^2}\sum_{t=1}^T\sum_{i\in \mathcal{N}}\hat{y}_i(t).\label{eq11.3}
    \end{alignat}\end{subequations}
\end{linenomath}
where inequality (\ref{eq11.3}) holds because $1+y\leq e^y$ and $M_t \leq b$. 

On the other hand, define $\mathcal{A}_b^*$ as the best location index subset with $b$ elements. 
Then,

\begin{linenomath}\begin{subequations}\label{eq12}
\postdisplaypenalty=0\begin{flalign}
    \ln{\frac{W_{T+1}}{W_1}} &\geq \ln{\frac{\sum_{j\in \mathcal{A}_b^*}w_j(T+1)}{W_1}}\label{eq12.1}\\ 
    &\geq \frac{\sum_{j\in \mathcal{A}_b^*}\ln w_j(T+1)}{b} - \ln\frac{N}{b}\label{eq12.2}\\
    &\geq \frac{\eta }{N}\sum_{j\in \mathcal{A}_b^*}\sum_{t: j\in S^c_0(t)}\hat{y}_j(t) - \ln \frac{N}{b}.\label{eq12.3}
\end{flalign}\end{subequations}
\end{linenomath}
where inequality (\ref{eq12.1}) holds because $\mathcal{A}^*_b \subseteq \mathcal{N}$, inequality (\ref{eq12.2}) comes from the inequality of arithmetic and geometric means, i.e. $\frac{1}{b}\sum_{j=1}^b y_j \geq \left(\prod_{j=1}^{b} y_j \right)^{\frac{1}{b}}$, and inequality (\ref{eq12.3}) is obtained by recursively applying step 15 of Algorithm 1, which results in equality \eqref{eq:theorem5.1_appendix}:

\begin{linenomath}
\begin{equation}\label{eq:theorem5.1_appendix}
\begin{aligned}
w_j(T+1) &= \exp\left((b \eta/N)\sum_{t:j\in S_0^c(t)}\hat{y}_j(t)\right).\\
\end{aligned}
\end{equation}
\end{linenomath}

Note that we also have

\begin{linenomath}\begin{subequations}\label{eq13}
    \postdisplaypenalty=0\begin{flalign}
    \sum_{j\in \mathcal{A}_b^*}\sum_{t:j\in S_0(t)}\hat{y}_j(t) 
    &\leq \sum_{t=1}^T\sum_{i\in S_0(t)}y_j(t)\label{eq13.1}\\
    &\leq \frac{1}{1-\eta}\sum_{t=1}^T\sum_{i\in S_0(t)}y_j(t)\label{eq13.2}
    \end{flalign}\end{subequations}
\end{linenomath}
where inequality (\ref{eq13.1}) is due to the fact that $\hat{y}_j(t) = y_j(t), \forall j \in S_0(t)$, and the last inequality (\ref{eq13.2}) holds because $\eta \in (0,1]$.

Combining \eqref{eq11.3}, \eqref{eq12.3}, \eqref{eq13.1}, and \eqref{eq13.2}, we have:

\begin{linenomath}\begin{subequations}\label{eq14}
\postdisplaypenalty=0\begin{flalign}
    &\sum_{j\in \mathcal{A}_b^*}\sum_{t:j\in S_0^c(t)}\hat{y}_j(t) + \sum_{j\in \mathcal{A}_b^*}\sum_{t:j\in S_0(t)}\hat{y}_j(t) - \frac{N}{\eta}\ln\frac{N}{b}\\
    &\leq \frac{1}{(1-\eta)}G^J_{\textrm{Exp3.M-VP}}(T) + \frac{(e-2)\eta b}{(1-\eta)N} \sum_{t=1}^T\sum_{i\in \mathcal{N}}\hat{y}_i(t)
\end{flalign}\end{subequations}
\end{linenomath}

Taking expectations of both sides of inequality (\ref{eq14}), we obtain

\begin{linenomath}\begin{subequations}\label{eq28_theorm_5_1_1}
\postdisplaypenalty=0\begin{flalign}
    & \sum_{j\in \mathcal{A}_b^*}\sum_{t:j\in S_0^c(t)}\hat{y}_j(t) + \sum_{j\in \mathcal{A}_b^*}\sum_{t:j\in S_0(t)}\hat{y}_j(t) - \frac{N}{\eta}\ln\frac{N}{b} \label{eq28.1_theorm_5_1_1}\\
    &\leq \frac{1}{(1-\eta)}E\left[G^J_{\textrm{Exp3.M-VP}}(T)\right] + \frac{(e-2)\eta b}{(1-\eta)N} \sum_{t=1}^T\sum_{i\in \mathcal{N}}y_i(t)\label{eq28.2_theorm_5_1_1}\\
    &\leq \frac{1}{(1-\eta)}E\left[G^J_{\textrm{Exp3.M-VP}}(T)\right] + \frac{(e-2)\eta b}{(1-\eta)a}G_{\max}(T),
\end{flalign}\end{subequations}
\end{linenomath}
where inequality (\ref{eq28.2_theorm_5_1_1}) uses the fact that $E[\hat{y}_i(t)|S(1),...,S(t-1)] = y_i(t)$, and 

\begin{linenomath}\begin{equation}
\sum_{t=1}^T\sum_{i\in \mathcal{N}}y_i(t) \leq \frac{N}{a}G_{\max}(T).
\end{equation}
\end{linenomath}

Since $\mathcal{A}^*_b = \cup_{M_t}\mathcal{A}^*_{M_t}$ trivially holds, we have

\begin{linenomath}
\begin{align}\label{eq31}
     G_{\max}(T) - \frac{N}{\eta}\ln\frac{N}{b} \leq& \sum_{j\in \mathcal{A}_b^*}\sum_{t:j\in S_0^c(t)}\hat{y}_j(t) \\
     &+ \sum_{j\in \mathcal{A}_b^*}\sum_{t:j\in S_0(t)}\hat{y}_j(t)- \frac{N}{\eta}\ln\frac{N}{b}
\end{align}
\end{linenomath}

Therefore, by combining \eqref{eq28_theorm_5_1_1} and \eqref{eq31}, we obtain the inequality stated in the Theorem \ref{THE3.1}.
\end{proof}

\section{Proof of Corollary 1.1}
\label{appendix cor511}

\begin{proof}
For any $T>0$, we have $G_{\max}(T) \leq bT$. If $bT \leq \sqrt{
\frac{N a  \ln(N/b)}{a+(e-2)b}}$, then the bound is trivial since the expected regret cannot be more than $bT$. Otherwise, by Theorem \ref{THE3.1}, the expected regret is at most
$$
2\sqrt{\left(1 + (e-2)\frac{b}{a}\right)}\sqrt{bTN\ln \frac{N}{b}}
$$
by plugging in $\eta = \sqrt{
\frac{N a  \ln(N/b)}{(a+(e-2)b)b T}}$.
\end{proof}

\section{Proof of Theorem \ref{THE4.1}} 
\label{appendix4}
\begin{proof}

Note that  

\begin{linenomath}\begin{subequations}\label{eq30}
    \begin{alignat}{2}
        \Bar{r}_{\infty} 
        &=&& 1 - \underset{T\rightarrow \infty}{\lim\inf} E\left[\frac{1}{T}G^J_{\textrm{Exp3.M-VP}}(T)\right]\label{eq30.1}\\
        &\leq&& 1 - \underset{T\rightarrow \infty}{\lim\inf} \frac{1}{T}\Bigg(G_{\textrm{max}}(T) - 2\sqrt{\left(1 + (e-2)\frac{b}{a}\right)}\notag\\
        &&&\sqrt{bTN\ln \frac{N}{b}}\Bigg) \label{eq30.2}\\
        &=&& 1 - \underset{T\rightarrow \infty}{\lim\inf} \frac{1}{T} G_{\textrm{max}}(T)\label{eq30.3}\\
        &\leq&& \frac{N-a}{N}\label{eq30.4}
    \end{alignat}\end{subequations}
\end{linenomath}

\noindent
for any policy $\pi$ of the attacker,
where the last inequality (\ref{eq30.4}) comes from the fact that $G_{\textrm{max}}(T) \geq \frac{Ta}{N}$ for any defender's policy $\gamma$.

Under the greedy policy we have $\hat{\beta}_{g(t)}(t)\leq \frac{b}{N}$, which implies $r(t) \geq \frac{N-b}{N}$ for any $t$. Therefore by using the greedy policy $\pi_{\textrm{greedy}}$, we have
$\bar{r}_{\infty}\geq \frac{N-b}{N}$.


\end{proof}

\section{Proof of Lemma \ref{LEMMA5.2}}
\label{appendix5}
The proof of Lemma \ref{LEMMA5.2} is an extension of the proof of Lemma 4 in \cite{wang2015learning}. The main difference is that the matrix $\mathbf{H}$ is now an $N \times \left|\mathcal{C}(\mathcal{N},a)\right|$ matrix compared to the one in the original proof which is $N \times N$.

\begin{proof}

1) The problem (\ref{eq21}) is equivalent to 

\begin{linenomath}\begin{equation}
    \underset{d\in \Delta_N}{\max}\ \underset{u\in \Delta_N}{\min} \sum_{n=1}^{\left|\mathcal{C}(\mathcal{N},a)\right|} \sum_{k=1}^N \left(\mu_k d_k - \sum_{j\in J_n} \mu_j d_j \right)u_n
\end{equation}\end{linenomath}

\noindent
which can be rewritten in matrix form:

\begin{linenomath}\begin{equation}
    \underset{\mathbf{d}\in \Delta_N}{\max}\ \underset{\mathbf{u}\in \Delta_N}{\min} \mathbf{d}^\intercal \mathbf{H} \mathbf{u},
\end{equation}\end{linenomath}

\noindent
where $\intercal$ denotes the transpose, and 

\begin{linenomath}
\[
\mathbf{H} = 
\begin{bmatrix}
&0, &\mu_1,&...&\mu_1\\
&&...&&\\
&0,&\mu_{a},&...&\mu_{a}\\
&\mu_{a+1}, &... &&\mu_{a+1}\\
&&...&&\\
&\mu_{N-a+1},&...&&0\\
&&...&&\\
&\mu_N,&...&&0\\
\end{bmatrix}
\]
\end{linenomath}

\noindent
where each column $j$ represents one set $\mathcal{S} \subseteq \mathcal{C}(\mathcal{N},a)$ such that for all $i\in \mathcal{S}$, $H_{ij} = 0$ and for all $i\in \mathcal{N}\setminus \mathcal{S}$, $H_{ij} = \mu_i$. The remaining proof is the same as the original proof.

Now consider a zero-sum game with the payoff matrices for the row and the column players being $\mathbf{H}$ and $-\mathbf{H}$, whose mixed strategy vectors are $\mathbf{d}$ and $\mathbf{u}$, respectively. Any optimal solution $\mathbf{d}^*$ to the problem (\ref{eq21}) is a Nash equilibrium strategy for the row player, and by the indifference condition, we obtain for any $j\in \textrm{supp}(\mathbf{d}^*)$,

\begin{linenomath}\begin{equation}
    \mathop{\sum_{k \neq j}}{k\in \textrm{supp}(\mathbf{d}^*)} = \textrm{Const.},
\end{equation}\end{linenomath}

\noindent
which implies $\mu_k d_k^* = \mu_j d_j^*$ for any $k,j \in \textrm{supp}(\mathbf{d}^*)$.

2) The second part of the Lemma is proved by contradiction. Assume that there exist $i\in \textrm{supp}(\mathbf{d}^*)$ and $j\in \mathcal{N}\setminus\textrm{supp}(\mathbf{d}^*)$ such that $\mu_j > \mu_i$. Let $o$ be a constant such that $o = \mu_k d_k^*$ for any $k \in \textrm{supp}(\mathbf{d}^*)$.  Then consider a feasible solution $\mathbf{d}$, where $d_k=0$ for all $k\in \left((\mathcal{N}\setminus \textrm{supp}(\mathbf{d}^*)) \setminus \{j\} \right) \cup \{i\} $, and $d_k = d_k^* + \epsilon $ for all $k\in \left(\textrm{supp}(\mathbf{d}^*)\setminus\{i\} \right) \cup \{j\}$, with $\epsilon = d_i^*(1 - \mu_i/\mu_j)/K^*$, which yields a higher objective value.

\end{proof}

\section{Proof of Lemma \ref{LEMMA5.3}}
\label{appendix6}
\begin{proof}

Consider the following linear program:

\begin{linenomath}
\postdisplaypenalty=0
\begin{subequations}\label{eq27}
    \begin{align}
        \underset{\mathbf{l},p}{\textrm{minimize}}\ & p \\
        \textrm{s.t.}\ & p + \mu_k l_k \leq \mu_k,\\
        & \sum_k l_k \leq b,\\
        & l_k \leq 1,\\
        & l_k \geq 0.
    \end{align}
\end{subequations}
\end{linenomath}

\noindent
It is easy to see that problem (\ref{eq26}) is lower bounded by the problem (\ref{eq27}).

Then the dual of the program (\ref{eq27}) can be written as 

\begin{linenomath}
\postdisplaypenalty=0
\begin{subequations}
    \label{eq28}
    \begin{align}
        \underset{\mathbf{d},q}{\textrm{maximize}}\ & \sum_{k=1}^N \mu_k d_k - q\\
        \textrm{s.t.}\ &\sum_{k=1}^N \mu_k d_k \leq \frac{Nq}{b},\\
        &\sum_{k=1}^N d_k = 1,\\
        & d_k\geq 0.
    \end{align}
\end{subequations}\end{linenomath}

Note that program (\ref{eq28}) is equivalent to the following problem 

\begin{linenomath}\begin{equation}
\label{eq29}
    \underset{\mathbf{d} \in \Delta_N }{\text{maximize}}\  \sum_{k=1}^{N} \mu_k d_k - \underset{J\in \mathcal{C}(\mathcal{N},b)}{\max}\sum_{j\in J} \mu_j d_j,
\end{equation}\end{linenomath}

\noindent
which is essentially problem (\ref{eq21}), except for changing the set $\mathcal{C}(\mathcal{N},a)$ to $\mathcal{C}(\mathcal{N},b)$. Therefore problem (\ref{eq29}) has the optimal value $\frac{K^*-b}{\sum_{k=1}^{K^*}\mu_k}$, which provides us with the lower bound.
\end{proof}

\end{appendices}

\bibliographystyle{IEEEtran}
\bibliography{refs}

\end{document}